\documentclass[sigconf]{aamas}
\usepackage{graphicx} 
\usepackage[utf8]{inputenc}
\usepackage{enumitem}
\usepackage{xcolor}
\usepackage[capitalize]{cleveref}
\usepackage[normalem]{ulem}
\usepackage{xspace}

\usepackage{caption}
\usepackage{subcaption}

\newtheorem{theorem}{\bf Theorem} 

\theoremstyle{definition}
\newtheorem{definition}{Definition}

\usepackage{tikz}
\usetikzlibrary{arrows.meta, positioning, calc}

\newcommand*{\harmonic}{\mathbb{H}\xspace}

\newcommand*{\oldharmonic}{\mathcal{H}\xspace}

\newcommand*{\hname}{modified harmonic mean operator\xspace} 

\newcommand*{\rlearning}{$R$-Learning\xspace}

 \newcommand*{\harmonicr}{Harmonic $R$-Learning\xspace}

\usepackage{balance} 

    \setcopyright{none}
    \acmConference[ALA '26]%
    {Proc.\@ of the Adaptive and Learning Agents Workshop (ALA 2026)}%
    {May 25 -- 26, 2026}%
    {Paphos, Cyprus, https://alaworkshop2026.github.io/}%
    {Aydeniz, Delgrange, Mohammedalamen, Yang (eds.)}%
    \copyrightyear{2026}
    \acmYear{2026}
    \acmDOI{}
    \acmPrice{}
    \acmISBN{}
\acmSubmissionID{57}

\title[A Mean to No End]{A Mean to No End: Average Reward Reinforcement Learning in Semi-Markov Settings}

\title[]{Average Reward Reinforcement Learning in Semi-Markov Settings Using a Harmonic Mean}

\title[]{A Harmonic Mean Formulation of Average Reward Reinforcement Learning in SMDPs}

\author{Erel Shtossel}
\affiliation{
  \institution{Bar Ilan University}
  \city{Ramat Gan}
  \country{Israel}}

\author{Alicia Vidler}
\affiliation{
  \institution{Bar Ilan University}
  \city{Ramat Gan}
  \country{Israel}
\email{aliciavidler@gmail.com}}

\author{Uri Shaham}
\affiliation{
  \institution{Bar Ilan University}
  \city{Ramat Gan}
  \country{Israel}}

\author{Gal A. Kaminka}
\affiliation{
  \institution{Bar Ilan University}
  \city{Ramat Gan}
  \country{Israel}
\email{galk@cs.biu.ac.il}}

\begin{abstract}
Recent research has revived and amplified interest in algorithms for undiscounted average reward reinforcement learning in infinite-horizon, non-episodic (continuing) tasks.
Semi-Markov decision processes (SMDPs) are of particular interest. In SMDPs, discrete actions stochastically generate both \emph{rewards} and \emph{durations}, and the objective is to optimize the average \emph{reward rate}.
Existing algorithms approach this by optimizing the ratio of rewards to durations.
However, when rewards and durations are non-stationary (in the infinite horizon), this can be incorrect.
This paper presents a novel \hname that correctly computes reward rates even under such conditions.
This yields model-free learning algorithms that can work with SMDPs, while maintaining robustness to non-stationary reward and duration distributions over time. We prove theoretical properties of the \hname, and empirically demonstrate its efficacy in comparison to existing algorithms.

\end{abstract}


\begin{document}
\sloppy

\pagestyle{fancy}
\fancyhead{}


\maketitle

\section{Introduction}
Recent research has revived and amplified interest in undiscounted average reward reinforcement learning~\cite{Schwartz93,mahadevan1996average,dewanto20}. Motivated by real-world applications, average reward reinforcement learning (ARL) addresses infinite-horizon, non-episodic (continuing) tasks. While less popular than the familiar discounted-rewards formulation for reinforcement learning, it is nonetheless a well known optimization criteria~\cite{puterman1994markov}.

While ARL has been studied from multiple perspectives in Markov Decision Processes (MDPs), there are almost no investigations
of ARL in semi-Markov settings (SMDP), where discrete actions have \textit{durations}.
In every decision step, an action is selected. It then takes a (stochastically) varying duration before the next state is known, and the rewards received.  The objective of ARL in SMDPs is to optimize the policy's average \emph{reward rate}.

To the best of our knowledge, only two algorithms were proposed for ARL in SMDP settings.
The first algorithm, SMART~\cite{smart}, uses the \textit{long-term sample average} to track the optimal reward rate. This makes it brittle when the rate is non-stationary, i.e., rewards and/or durations are not stationary. Relaxed-SMART~\cite{r-smart} handles the non-stationary by approximating the reward rate as the ratio of the \textit{Exponential Moving Average} (EMA) of the rewards, to the EMA of the durations.  This approximation is perfect when the rewards and durations are independent of each other, but faces limitations otherwise (as we show in~\cref{sec:harmonicmath}).

We note that rather than approximating the average rate as the ratio of arithmetic average of the rewards to the arithmetic average of the durations, the average rate can be directly calculated using the harmonic mean, i.e., as the harmonic average of the reward rates received in different steps. For example, the harmonic mean is the correct averaging operator to use in averaging velocities (distance/time), or throughput (units/time). 
Unfortunately, the harmonic mean is not well-defined for zero or mixed-sign rewards, and therefore ill-suited to general reinforcement learning (see \cref{sec:harmonicmath}).

To address this, we first present a novel \hname that correctly computes reward rates (including when rewards are negative or zero), and admits stochastic approximation: an exponential moving harmonic mean, rather than the exponential moving arithmetic mean so commonly used in reinforcement-learning literature. We prove theoretical properties of \hname, showing that it generalizes the familiar harmonic mean,  and that it obeys axiomatic definitions of means. We also show that it is not a member of the class of quasi-linear means, which includes arithmetic, geometric, and harmonic means.

Use of the \hname yields a novel model-free reinforcement learning algorithm for SMDPs---\harmonicr---that directly generalizes existing algorithms, and does not suffer from the same weaknesses as SMART or Relaxed-SMART.
We discuss the conditions under which \hname is better than the alternative (approximation by the ratio).
We empirically demonstrate the improvements offered by \harmonicr over existing algorithms in two domains with non-stationary, volatile reward rates: a simple two-state SMDP, and Bitcoin trading (using real data).

\vspace{-8pt}
\section{Background: Average Reward RL}\label{sec:motivation}
We begin with a brief reminder of the problem of optimizing the average undiscounted rewards in Markov Decision Processes (MDPs), and how it has been addressed in reinforcement learning literature (in~\cref{sec:mdp-rlearning}).  We then discuss such learning in Semi-Markov Decision Processes (SMDPs), where transition durations (\emph{sojourn} times) must
be taken into account (\cref{sec:smdp-smart}).

\subsection{Maximizing Average Undiscounted Reward}\label{sec:mdp-rlearning}
Markov Decision Processes (MDPs) are a popular formulation of sequential decision making problems. An infinite-horizon MDP is defined as a tuple $\langle \mathcal{S},\mathcal{A},\mathcal{T},\mathcal{R} \rangle$.
$\mathcal{S}$ is a set of states (at any time $t$, the agent is in a state $s_t\in\mathcal{S}$.).
$\mathcal{A}$ is a set of actions that the agent can take to transition between states. $\mathcal{T}:\mathcal{S}\times\mathcal{A}\times\mathcal{S}\mapsto [0,1]$ is a transition probability function, determining the probability that an action $a\in\mathcal{A}$, taken in a state $s\in\mathcal{S}$, will result in a transition to a new state $s'\in\mathcal{S}$.

The reward function kernel $\mathcal{R}:\mathbb{N}\times\mathcal{S}\times\mathcal{A}\times\mathcal{S}\mapsto \mathcal{P}(\mathbb{R})$,
maps a transition taken on step $t\in\mathbb{N}$ to a  probability distribution over rewards in $\mathbb{R}$. We write $\mathcal{R}_t(s,a,s')$ to define the distribution from which the system samples the reward received by the agent on taking the action $a$ in state $s$, and transitioning to the state $s'$ at time $t$. 
We use $R_t$  to denote the random variable drawn from this distribution.
The use of the index $t$ in the definition allows discussing non-stationary rewards, where the distribution of rewards may change with $t$.

A solution to an MDP is a
deterministic policy $\pi:\mathcal{S}\mapsto\mathcal{A}$, which assigns an action to each state, 
i.e., $\forall \, s_t\in\mathcal{S}, \exists a_t\in\mathcal{A} \; \text{such that } \pi(s_t) = a_t$.









The value of $R_t$ is determined considering both $\mathcal{T}$ (which determines the probability of all possible next states $s_{t+1}$, given the action) and the reward distribution $\mathcal{R}_t(s,a,s')$ (which determines the reward for reaching any specific $s_{t+1}$ from $s_t$, given the action determined by the policy $\pi$).
On taking an action $a_t$ in state $s_t$ and transitioning to $s_{t+1}$, the agent receives a realized reward
\[
r_t \sim \mathcal{R}_t(s_t,a_t,s_{t+1}),
\]



\paragraph{\bf Optimality Criteria}
Given a starting state $s_0$, a policy $\pi^*$ is \textit{optimal} in some optimality criterion $O(\cdot)$, if $O(\pi^*(s_0))\geq O(\pi(s_0)), \forall\pi\in \Pi$, where $\Pi$ is the set of all possible policies.
Many readers are familiar with the \emph{discounted rewards} criterion.

\textit{Note:} In the following equations we make use of the law of iterated expectations, namely; that we state $E[R_t]$ instead of $E_\pi[R_t]$ relying on iterated expectation over trajectories and the inner on $R_t$, conditioned on those trajectories.


\begin{equation}
	O_d(\pi(s_0)) := \lim_{T\to\infty} \mathbb{E}\left[ \sum_{t=0}^{T-1} \gamma^t  R_t \right]
	\label{eq:discounted-reward}
\end{equation}




However, a different criterion, often useful in infinite-horizon tasks, is the \emph{average undiscounted rewards} criterion~\cite{Schwartz93,sutton2018reinforcement,puterman1994markov,mahadevan1996average,dewanto20}. It is particularly useful in cycling and continuing tasks, where an optimal policy repeatedly goes through the same states:



\begin{equation}
	O_a(\pi(s_0)) := \lim_{T\to\infty} \mathbb{E}\left[ \frac{1}{T}\sum_{t=0}^{T-1}  R_t\right]
	\label{eq:average-reward}
\end{equation}

\paragraph{\bf Average Rewards Reinforcement Learning (ARL)}
In a seminal paper, \citet{Schwartz93} presented the \rlearning algorithm, the basic tabular model-free reinforcement learning algorithm that optimizes policies for the average reward criterion (\cref{eq:average-reward}). The paper introduces
the notation $\rho^\pi$ to denote the average reward $O_a(\pi(s_0))$ of a policy $\pi$. 
The \rlearning algorithm seeks to optimize $\rho$, by learning an optimal policy. \rlearning is structurally identical to $Q$-Learning (e.g., in terms of selecting actions).




 It then takes two update steps:

First, it updates the $Q$ table entry using \cref{eq:r-update-q}:
\begin{equation}
	Q_{t+1}(s_t,a_t)\gets \alpha\left( r_t - \rho_t + \max_{a'} Q_t(s_{t+1},a') - Q_{t}(s_t,a_t)\right).
	\label{eq:r-update-q}
\end{equation}
Then, and \emph{only if the action $a_t$ was selected on-policy} (i.e., was not an exploratory action), \rlearning also  updates the parameter $\rho$, which maintains the average reward of the learned policy across all states:
\begin{equation}
	\rho_{t+1} \gets \beta\left( r_t + \max_{a'} Q_t(s_{t+1},a') - \max_{a} Q_{t+1}(s_{t},a) - \rho_t\right).
	\label{eq:r-update-rho}
\end{equation}

In essence, the $Q$ table is updated using the \emph{marginal reward} with respect to the average policy reward $\rho$, which is actually the optimization goal of the learning process. The reason for this is to prevent $Q$ from increasing without bound. Instead, $Q$ is bounded by the difference with $\rho$ (see~\cite{Schwartz93,mahadevan1996average,dewanto20,sutton2018reinforcement} for details).

Over the years, there have been investigations of \rlearning variants and novel algorithms for model-based average reward learning~\cite{h-learning}, deep neural-network approaches~\cite{deeprlearning20, apo}, and theoretical investigations. \citet{mahadevan1996average} provides an early synthesis; \citet{dewanto20} provides a recent comprehensive review.

The use of two stochastic approximators in \rlearning, $\alpha$ in \cref{eq:r-update-q} and $\beta$ in \cref{eq:r-update-rho} is a noted weakness of \rlearning~\cite{dewanto20}. Many variants focus on modifying the $\rho$ update steps to improve the practical performance of \rlearning.  However, to the best of our knowledge, all variants---with one exception, below---have assumed that the underlying environment is an MDP.
This allows them to estimate the average reward $\rho$ using the standard arithmetic exponential moving average. 


\subsection{Undiscounted Rewards in Semi-Markov Decision Processes}\label{sec:smdp-smart}

Many real world decision processes unfold over infinite horizons with unknown stopping times, where neither the termination point nor the duration between actions is fixed or known in advance. Such settings are naturally modeled by \emph{semi-Markov decision processes} (SMDPs), which generalize MDPs by introducing stochastic, state-action-dependent holding times (\emph{sojourn} times) between decisions.

An SMDP is defined as a tuple $\langle \mathcal{S},\mathcal{A},\mathcal{T},\mathcal{R}, \Omega \rangle$, where
$\mathcal{S}$, $\mathcal{A}$, $\mathcal{T}$, and $\mathcal{R}$ are defined as above (MDP).
SMDPs relax the unit-time MDP assumption by admitting \emph{stochastic holding times}, captured
by the function 
$\Omega:\mathbb{N}\times\mathcal{S}\times\mathcal{A}\times\mathcal{S}\mapsto \mathcal{P}(\mathbb{R}^+)$.
Analogous to the definition of the distribution $\mathcal{R}_t(s,a,s')$, and the random variable $R_t$, we define the distribution $\Omega_t(s,a,s')$ and random variable $\varUpsilon_t$.

When an action $a\in A$ is taken in state $s$ at time $t\in\mathbb{N}$, the new state $s'$ is stochastically determined by $\mathcal{T}$, and the \textit{sojourn time} $\tau_t$ is sampled from  by $\Omega_t(s,a,s')$, analogous to how the reward $r_t$ is drawn.  

The average reward optimization criterion is a natural fit for SMDPs. Rewards are naturally weighted by the time, and are therefore best understood as a reward rate (lump sum reward, divided by the sojourn time). 
It is then straightforward to consider $\rho$ as the \textit{average reward rate}---reward \emph{per unit time}---of a stationary policy $\pi$ (\cref{eq:rhorate}), and the optimal policy yielding an optimal $\rho$: $\rho^{*}=\sup_{\pi}\rho^\pi$. 


\begin{equation}
    	\rho^\pi\;=\;\lim_{T\to\infty} \mathbb{E}\!\left[\frac{\sum_{t=0}^{T-1} R_t}{\sum_{t=0}^{T-1} \varUpsilon_t}\right].
	\label{eq:rhorate}
\end{equation}
Estimation of the average reward \emph{rate} $\rho$ is the key to model-free average reward reinforcement learning in SMDP settings.

\paragraph{\bf Average Reward RL in SMDPs: \rlearning with Reward Rates}
To the best of our knowledge, only two reinforcement learning algorithms for average reward optimization in SMDPs: SMART and Relaxed-SMART~\cite{r-smart}.

SMART~\cite{smart} builds on \rlearning, modifying it as follows: First, the \rlearning $Q$ update step (\cref{eq:r-update-q}) is modified to compute the marginal update to $Q$ based on the average \emph{rate} $\rho$, rather than the average reward. This is done using $\tau_t$:  
\begin{equation}
	Q_{t+1}(s_t,a_t)\gets \alpha\left( r_t - \rho_t\cdot\tau_t + \max_{a'} Q_t(s_{t+1},a') - Q_{t}(s_t,a_t)\right).
	\label{eq:smart-update-q}
\end{equation}
Essentially, this allows a comparison between the reward received $r_t$ (with holding time $\tau_t$), and the expected reward, given the reward rate $\rho_t$. 

If (and only if) action $a_t$ was selected on-policy, SMART also incrementally updates $\rho$ so as to track the average reward rate of the policy. It does this by tracking the total sum of the rewards received so far $X_t$, and the total sum of the holding times received so far $Y_t$:
\begin{flalign}
	Y_{t+1}\gets Y_t + \tau_t, & \qquad X_{t+1}\gets X_t + r_t, \notag \\
	\rho_{t+1} \gets \frac{X_{t+1}}{Y_{t+1}}.
	\label{eq:smart-update-rho}
\end{flalign}
This calculation uses the long-term sample average rate $\frac{\sum_{i=0}^{t} r_t}{\sum_{i=0}^{t} \tau_t}$ to approximate the expected reward rate $\rho$ in \cref{eq:rhorate}. 

The choice to approximate the average reward rate $\rho$ by the time sample average (\cref{eq:smart-update-rho}) is unusual in the context of reinforcement learning, because it makes a strong assumption as to the stationarity of the rewards and duration.  As long as the rewards and holding times are stationary, the approximation is generally correct.  Otherwise, \cref{eq:smart-update-rho} can be dominated by rare outcomes (outliers), expand dramatically with unbounded rewards (or holding times), or, be affected by latent regime shifts in rewards or holding times.  SMART will then fail to estimate $\rho$ correctly.

The Relaxed-SMART algorithm~\cite{r-smart} attempts to address this strong assumption of stationary rewards and holding times,  by estimating $\rho$ in a different manner, which facilitates
assigning more weight to recent values. 
Rather than using the long-term average (\cref{eq:smart-update-rho}), Relaxed-SMART uses the ratio of average rewards to average durations (\cref{eq:rsmart-update-rho}):
\begin{flalign}
	\bar{\tau}_{t+1}\gets \beta \bar{\tau}_t + (1-\beta) \tau_t, & \qquad \bar{r}_{t+1}\gets \beta \bar{r}_t + (1-\beta) r_t, \notag \\
	\rho_{t+1} \gets \frac{\bar{r}_{t+1}}{\bar{\tau}_{t+1}}.
	\label{eq:rsmart-update-rho}
\end{flalign}

\vspace{-6pt}
\section{\harmonicr}
\label{sec:harmonicr}

\label{sec:harmonicmath}
We develop a novel \rlearning variant that can be used with reward rates, i.e., in SMDP settings.
\cref{sec:old-harmonic-mean} presents the idea of using the harmonic mean as a an alternative basis for estimating the average rate from the reward and holding time generated from each step.  We show that this approach is not reliant on the assumption that rewards and holding times are independent of each other, as previous approaches assume. Then, in~\cref{sec:newharmonic} we
address the limitations of using the standard harmonic mean (e.g., not defined for rates of mixed signs). We introduce a novel alternative, the \hname\ $\harmonic$, and investigate its axiomatic properties.
\cref{sec:harmonic-agent} then shows how it can be used in \rlearning.

\subsection{The Harmonic Mean (Is Not Enough)}
\label{sec:old-harmonic-mean}
We observe that there is a different approach to averaging rates, using the harmonic mean.  The \emph{harmonic} mean (\cref{def:oldharmonic} below) is the correct averaging calculation to be used with rates and flows (and indeed
used in engineering and science in areas focusing on such variables, e.g., electric resistance, liquid flow rates, etc.). 

\begin{definition}[Harmonic mean, $\oldharmonic$]
\label{def:oldharmonic}
Let $X = (x_1,\ldots,x_n)$ be a multi-set of non-zero numbers $x_i\in\mathbb{R}$.
The \emph{harmonic mean} of $X$ is given by:
\begin{equation}
   \oldharmonic(X) \;=\; \oldharmonic(x_1,\dots,x_n)\;=\;\frac{n}{\sum_{k=1}^n \frac{1}{x_k}}.
    \label{eq:harmonicMean}
\end{equation}
\end{definition}

For example, given data about the speeds at which a fixed distance $D$ is traveled---first at 20kph and then again at 40kph---one cannot infer
the average speed is 30kph, as that would be incorrect: Traveling $D$ at 20kph takes twice the amount of time it takes to travel $D$ at 40kph. Instead, the correct average speed is 26.6kph, which is the harmonic mean of 20 and 40.

Based on this observation, we may want to estimate $\rho$ as the incrementally computed harmonic mean of the sampled reward rates $\dfrac{r_1}{\tau_1},\dfrac{r_2}{\tau_2},\ldots$, where each rate $\dfrac{r_t}{\tau_t}$ resulting from a step $t$ is used as single atomic datum ($x_t$ in~\cref{eq:harmonicMean}).

\label{sec:rate-matters}
Calculating $\rho$ without separating the aggregations of the reward and holding time can yield the same value as the ratio of averages in some specific cases. We compare the behavior of aggregating the values using the ratio of the arithmetic means vs. using the harmonic mean.

We begin with an intuition.  Suppose we have taken $n$ steps in the SMDP, resulting in a sequence of rewards $(r_1,r_2,\ldots r_n)$ and an associated sequence of holding times $(\tau_1,\tau_2,\ldots \tau_n)$.  
Intuitively, when the rewards and the holding times are completely independent of each other, then a change in the internal order of the sequence of rewards does not matter when assessing the average reward $\bar{r}$.  But if the rewards and holding times are not independent, then changing the order of reward (e.g., exchanging $r_1$ and $r_n$) necessitates changing the order of the holding times ($\tau_1$ and $\tau_n$), to preserve the coupling between the rewards and their duration.  The ratio of averages ignores this coupling and would produce an incorrect average rate, while the harmonic mean (which averages the rates $\dfrac{r_1}{\tau_1},\ldots \dfrac{r_n}{\tau_n}$ directly) respects this coupling.

\Cref{thm:rate-equivalence} formalizes this intuition. It shows that the harmonic mean and the ratio of arithmetic means will differ under specific conditions.

\begin{theorem}\label{thm:rate-equivalence}
Let $r_t>0$ and $\tau_t>0$ for $t=1,\dots,n$. Define the sample mean
\[
A(x) := \frac{1}{n}\sum_{t=1}^n x_t,
\]
and the (mean-based) covariance

\begin{equation}
\operatorname{Cov}(X,Y) := A(XY) - A(X)\,A(Y).
\label{eq:Cov0}    
\end{equation}

Define
\[
Q := \frac{A(r)}{A(\tau)} = \frac{\sum_{t=1}^n r_t}{\sum_{t=1}^n \tau_t},
\qquad
H := \frac{n}{\sum_{t=1}^n \frac{\tau_t}{r_t}} = \frac{1}{A(\tau/r)}.
\]
Then
\[
Q = H \iff \operatorname{Cov}\!\left(r,\frac{\tau}{r}\right)=0.
\]
Equivalently, $Q \neq H$ iff the covariance is nonzero.
\end{theorem}

\begin{proof}[Proof Sketch]
Using the definition of covariance,
\[
\operatorname{Cov}(X,Y)=A(XY)-A(X)A(Y)
\quad\Longleftrightarrow\quad \]
\[\
A(XY)=A(X)A(Y)+\operatorname{Cov}(X,Y).
\]
Take $X=r$ and $Y=\tau/r$. For generality, we drop the time $t$ subscript for the purpose of this proof in $r_t$ and since $r_t\neq 0$,
\[
XY = r\cdot(\tau/r)=\tau,
\]
hence
\[
A(\tau)=A(r)\,A(\tau/r)+\operatorname{Cov}\!\left(r,\frac{\tau}{r}\right).
\]
By the definition of $H$,
\[
H=\frac{n}{\sum_{t=1}^n \frac{\tau_t}{r_t}}=\frac{1}{A(\tau/r)}
\quad\Longrightarrow\quad
A(\tau/r)=\frac{1}{H}.
\]
Substituting into the previous identity gives
\[
A(\tau)=\frac{A(r)}{H}+\operatorname{Cov}\!\left(r,\frac{\tau}{r}\right),
\]
so
\[
H=\frac{A(r)}{A(\tau)-\operatorname{Cov}\!\left(r,\frac{\tau}{r}\right)}.
\]
Therefore,
\[
H=\frac{A(r)}{A(\tau)} \iff \operatorname{Cov}\!\left(r,\frac{\tau}{r}\right)=0.
\]
Since $Q=\frac{A(r)}{A(\tau)}$, this is exactly $Q=H \iff \operatorname{Cov}\!\left(r,\frac{\tau}{r}\right)=0$.
\end{proof}

In other words, as long as the relationship between rewards and time induces the Covariance in equation \ref{eq:Cov0} to be \textbf{non zero}, we expect to observe a difference in $\rho$ between SMART and Relaxed-Smart, relative to the $\harmonic$.  Under finite second moments, such non zero covariance between $r$ and $\frac{\tau}{r}$ implies that the variables are not independently distributed. In short, nonzero covariance forces the two mean measures to differ.

In other words, as long as the relation between the reward and time is non-linear, we expect to observe a difference in the $\rho$ value between SMART and Relaxed-SMART in comparison to $\harmonic$.

Unfortunately, the harmonic mean has strict domain restrictions: First, it is not defined for zero reward values (as they have no defined reciprocals).
Second, if values in $X$ have mixed signs, $\oldharmonic$ loses critical properties expected of means. Most glaringly, it loses the \emph{internality} property (value between the minimum and maximum datum, see below) expected with any measure of central tendency: $\oldharmonic(-20, 40)\;=\; -80$, for example.

These restrictions severely limit the applicability of the harmonic mean to the use of reinforcement learning using reward rates in the
model-free \rlearning algorithm, as in many environments, rewards can be zero (though their associated sojourn times are not zero), or negative.
Below, therefore, we propose a new mean operator, \hname, to address these restrictions.

\subsection{The \hname\ $\harmonic$}\label{sec:newharmonic}
We formalize a mixed-sign operator $\harmonic$ that admits zero and mixed-sign datums (\cref{def:newharmonic}).
It is the weighted arithmetic mean of (i) the zero-valued datums, and the harmonic means of (ii) the positive datums, and of (iii) the negative datums.

\begin{definition}[\hname~ $\harmonic$]
	\label{def:newharmonic}
Let $X = (x_1,\ldots,x_n)$ be a multi-set of numbers $x_i\in\mathbb{R}$.
Let $X^+, X^-, X^0$ be partitions of $X$ such that
 $X^+ = (x_i | x_i\in X, x_i > 0)$, $X^- = (x_i | x_i\in X, x_i < 0)$, and $X^0 = (x_i | x_i\in X, x_i = 0)$.
The \hname\ $\harmonic$ is defined as:
\begin{align}
	\harmonic(X) := &  \frac{|X^+|\cdot\oldharmonic(X^+) \;+\; |X^-|\cdot\oldharmonic(X^-) \;+\; |X^0|\cdot 0}{|X^+|+|X^-|+|X^0|} \notag \\
	= & \frac{|X^+|\cdot\oldharmonic(X^+) \;+\; |X^-|\cdot\oldharmonic(X^-)}{|X|},
	\label{eq:harmonic-mixed}
\end{align}
where $\oldharmonic$ is the harmonic mean (\cref{eq:harmonicMean}):

\begin{equation}
    \label{eq:harmonic-pos-neg}
    \oldharmonic(X^+)=\frac{|X^+|}{\sum\limits_{x_i\in X^+} \left( \frac{1}{x_i}\right)},\qquad
    \oldharmonic(X^-)=\frac{|X^-|}{\sum\limits_{x_j\in X^-}\left(\frac{1}{x_j}\right)}.
\end{equation}
\end{definition}

We consider theoretical properties of \hname.
There are several axiomatic definitions of means, or general \emph{mean-type aggregation functions}~\cite{bullen2003handbook,beliakov2007aggregation,singpurwalla20,gray2013meanaxiomaticsgeneralizationsapplications}, which we refer to as \emph{general means} for brevity.
We will show that the
\hname\ is a general mean, but not a quasi-arithmetic mean (the class that includes the familiar arithmetic, geometric, harmonic means, and many others).

%
%
%
%

\begin{definition}[General Mean]\label{def:general-mean}
Following~\cite{bullen2003handbook,aczel48}, an operator $\mathcal{M}$ is a \textit{general mean} if it satisfies
the following properties, with respect to all  multisets $X$:
\begin{description}
	\item[Internality.] $\min_i x_i \leq \mathcal{M}(x_1,\ldots,x_n) \leq \max_i x_i$.
	\item[Idempotence (reflexivity).] $\mathcal{M}(x,\ldots,x)=x$ for any $x$.
	\item[Symmetry (permutation invariance).] $\mathcal{M}$ depends only on the values in the finite multiset $X$, but not their order: for any permutation $\sigma$ of $(1,\dots,n)$,
	\[
	\mathcal{M}(a_1,\dots,a_n)=\mathcal{M}\big(a_{\sigma(1)},\dots,a_{\sigma(n)}\big).
	\]
	\item[Monotonicity.] If one datum value increases (others fixed), the mean should not decrease. Given $Y=(x_1,\ldots,x_i+k,\ldots,x_n)$, where $k>0$, $\mathcal{M}(Y)\geq \mathcal{M}(X)$.
\end{description}
\end{definition}

\noindent An important and familiar subclass of means is \emph{quasi-arithmetic means}, which obeys the above definition, but in addition, has one more property (\cref{def:mean-replacement}), where by any subgroup of datums $Y\in X$ may be replaced by their subgroup mean $\mathcal{M}(Y)$, without affecting the $\mathcal{M}(X)$~\cite[][as presented in~\cite{bullen2003handbook,aczel48}]{kolmogorov1930mean,nagumo1930klasse}.

\begin{definition}[$n$-Associativity (replacement invariance, consistency)]
	\label{def:mean-replacement}
	A mean operator $\mathcal{M}$ is \textit{consistent} if for any multiset $X=(x_1,\ldots,x_n)$, and any \emph{block} $B=(x_i,\ldots,x_j)\subset X$ (for any $1\leq i,j \leq n$), the values $x_i,\ldots,x_j$ can be replaced by $|B|$ copies of the block mean $\mathcal{M}(B)$, without changing $\mathcal{M}(X)$:
$$
\mathcal{M}(x_1,\ldots x_i,\ldots x_j,\ldots,x_n)=\mathcal{M}(x_1,\ldots\mathcal{M}(B),\ldots\mathcal{M}(B),\ldots x_n).
$$
\end{definition}
%

\noindent The familiar arithmetic mean is a quasi-arithmetic mean. 
The harmonic mean $\oldharmonic$ is a quasi-arithmetic mean under the domain restriction that either $\forall x\in X, x>0$ or $\forall x\in X, x<0$.
If the restriction is removed, the harmonic mean loses the internality property at the very least (for mixed signs), or becomes undefined (for zero values).

We show that where the domain restrictions for the harmonic mean are satisfied, $\harmonic$ generalizes $\oldharmonic$ (\cref{thm:generality}). We then use this result to show that $\harmonic$ is a general mean (\cref{thm:mean}), but not a quasi-arithmetic mean (\cref{thm:noconsistency}). For brevity, we present proof sketches only.

\begin{theorem}[$\harmonic$ generalizes $\oldharmonic$]\label{thm:generality}
	Given a multiset of positive (respectively, negative) values $X^+$ ($X^-$), $\harmonic(X^+)=\oldharmonic(X^+)$ ($\harmonic(X^-)=\oldharmonic(X^-)$.
\end{theorem}
\begin{proof}[Proof sketch]
	Assume $X=(x_1,\ldots,x_n)$, where $\forall i, x_i>0.$
	Then by~\cref{def:newharmonic},
	\begin{align}
		\harmonic(X) = & \frac{|X^+|\cdot\oldharmonic(X^+) \;+\; |X^-|\cdot\oldharmonic(X^-)}{|X|} \tag{by \cref{eq:harmonic-mixed}.} \\
		= & \frac{|X|\cdot\oldharmonic(X)+0\cdot\oldharmonic(\emptyset)}{|X|} \tag{since $X^+=X.$} & \\
		= & \oldharmonic(X). \notag
	\end{align}
	The same argument holds when $\forall i, x_i<0.$
\end{proof}

\noindent
\cref{thm:mean} shows that unlike $\oldharmonic$, $\harmonic$ is a general mean across all values ($\forall i, x_i\in\mathbb{R}$):

\begin{theorem}\label{thm:mean}
 $\harmonic$ is a general mean.
\end{theorem}
\begin{proof}[Proof sketch]
We show $\harmonic$ satisfies the properties listed in \cref{def:general-mean}: internality, idempotence, permutation invariance, and monotonicity.
\paragraph{Internality (satisfied)}
				$\harmonic$ is bounded between the harmonic mean of the positives and that of the negatives. Each of these, by itself, maintains internality. So $\harmonic$ always lies between the dataset minimum and maximum.
\paragraph{Idempotence (satisfied)} If $x=0$, $\harmonic(x,\ldots,x)=0=x$. If $x>0$, or $x<0$, then $\harmonic=\oldharmonic$ (see \cref{thm:generality}) and thus maintains idempotence (this can also be proven directly). 

\paragraph{Symmetry (satisfied)} $\harmonic(X)$ is unchanged by reordering datums in $X$, as the partitioning of $X$ into the multisets $X^+, X^-, X^0$ is not dependent on the ordering.

\paragraph{Monotonicity (satisfied)} Note the denominator $|X|$ does not change. Pick any $x_i$ in $X^+$ (respectively, $X^-$). Add a constant $k>0$. If $x_i+k$ remained in $X^+$ ($X^-$), then by the monotonicity of $\oldharmonic$, $\oldharmonic(X^+)$ (or $\oldharmonic(X^-)$) will increase and thus $\harmonic(X)$ will increase. If $x_i+k=0$, then $x_i$ is eliminated from $X^-$. 
The negative-valued $\oldharmonic(X^-)$ will increase, and thus $\harmonic(X)$ will increase.

\vspace{6pt}\noindent
{\bf Conclusion.} As $\harmonic$ satisfies \cref{def:general-mean}, it is a general mean.
\end{proof}

The \hname\ $\harmonic$ generalizes the harmonic mean $\oldharmonic$ where the latter is defined, but $\harmonic$ is defined where $\oldharmonic$ is not (for zero-value datums). Also, unlike $\oldharmonic$, $\harmonic$ maintains the centrality property when datums of mixed signs are present.  It is therefore a general mean. However, it is not a quasi-arithmetic mean (\cref{thm:noconsistency}).

\begin{theorem}\label{thm:noconsistency}
 $\harmonic$ is not a quasi-arithmetic mean.
\end{theorem}
\begin{proof}[Proof sketch]
We note that consistency is maintained within pure-sign blocks, as a function of the consistency of $\oldharmonic$ with respect to pure-sign blocks.
However, if a block $B$ has mixed signs, consistency is not guaranteed.
Assume for contradiction that $\harmonic$ is a quasi-arithmetic mean. It is therefore consistent (per \cref{def:mean-replacement}).
Let $X=(1,1,-1,4)$, and $B=(1,-1)\subset X$. Then,
\[
\harmonic(1,1,-1,-4) = -0.3, \qquad \harmonic(1,-1) = 0.
\]
However, replacing $B$ by $\harmonic(B)$ yields $\harmonic(1,0,0,-4) = -0.75$.
This contradicts the assumption that $\harmonic$ is consistent, and hence it cannot be a quasi-arithmetic mean.
\end{proof}





Finally, it remains to show that $\harmonic$ maintains the property discussed in \cref{thm:rate-equivalence}, which distinguishes it from the ratio of averages. We show in \cref{thm:3collections-dependent} that if the distribution of all rewards is not independent of all times, then when rewards are decomposed into their positive, negative, and zero values, at least one of three reward-collections must also not be independent of time, hence maintaining the distinction. 

\begin{theorem}
If the distributions of all rewards are \textbf{not independent} of the distribution of all holding times, then when rewards are partitioned based on their sign (positive, negative, zero), at least one of three reward partition collections must also be dependent on the respective partition of holding times.
\label{thm:3collections-dependent}
\end{theorem}

\begin{proof}
Skipped due to space limitations. See Appendix~\ref{sec:full-proof}.
\end{proof}

\subsection{Using $\harmonic$ in Average Reward RL}
\label{sec:harmonic-agent}
We introduce a novel average reward RL algorithm, called \harmonicr \footnote{\hyperlink{https://github.com/erelon/SMDP_Agents/tree/master}{https://github.com/erelon/SMDP\_Agents}}.
The $Q$ update rule for \harmonicr is the same as for SMART ($\rho$ is multiplied by $\tau_t$, \cref{eq:smart-update-q}).
We develop an incremental approximator of $\harmonic$, to be used in updating the average reward rate $\rho$.




We use $p$ (respectively, $n$) to maintain the exponential moving \emph{arithmetic} means of the positive rewards (negative rewards, resp.). Analogously, we use $w_p,w_n,w_z$ to maintain the exponential moving \emph{arithmetic} approximation of $|H^+|, |H^-|, |H^0|$, resp.  We denote by $E^+, E^-$ the exponentially moving approximations of the harmonic means $\oldharmonic(H^+), \oldharmonic(H^-)$, resp.

We set $r$, the reciprocal of the rate $\dfrac{r_t}{\tau_t}$ (when $r_t \neq 0$)


$$r:=\begin{cases}
 0, & r_t=0, \\
 \dfrac{\tau_t}{r_t}, & \text{otherwise}.
\end{cases}
$$
Then, we set the reciprocals $p, n$ and the associated weights $w_p,w_n$, using the indicator function $\mathbf{1}$ (1 when the condition holds, 0 otherwise). We also set $w_z$:
\begin{flalign}
	p\gets \beta\cdot(\mathbf{1}\{r>0\}\cdot r - p), & \qquad w_p\gets \beta\cdot(\mathbf{1}\{r>0\} - w_p), \\
	n\gets \beta\cdot(\mathbf{1}\{r<0\}\cdot r - n), & \qquad w_n\gets \beta\cdot(\mathbf{1}\{r<0\} - w_n), \\
	& \qquad  w_z\gets \beta\cdot(\mathbf{1}\{r=0\} - w_z).
\end{flalign}
Now we can compute $E^+, E^-$ by inverting the reciprocals $n,p$ (accounting for the weight $w_p,w_n$)
\begin{flalign}
	E^+ \gets \begin{cases}
		0, & p=0, \\
		\dfrac{w_p}{p}, & \text{otherwise}.
	\end{cases}, & \qquad E^- \gets \begin{cases}
	0, & n=0, \\
	\dfrac{w_n}{n}, & \text{otherwise}.
	\end{cases}
\end{flalign}

Finally, the exponentially moving $\harmonic$ of $\rho$ is given by
\begin{equation}
	\rho_{t+1} \gets \frac{w_p\cdot E^+ + w_n\cdot E^-}{w_p+w_n+w_z}.
	\label{eq:harmonic-update-rho}
\end{equation}

The \harmonicr algorithm is an \rlearning variant, where $\rho$ is a stochastically-estimated average reward rate.  In that, it is completely analogous to \rlearning, but where the latter is appropriate only for MDP environments, \harmonicr is intended for use in the more general SMDP case. Note that by setting $\tau_t=1$, one gets back the familiar $\rho$ computation via the arithmetic mean (\cref{eq:r-update-rho}).
It is therefore a proper generalization of \rlearning.

\vspace{-6pt}
\section{Empirical Evaluation}\label{sec:eval}
In addition to the theoretical analysis of \hname, we also investigate it empirically in extensive experiments, using deceptively simple---and very challenging---SMDPs (\cref{sec:sim}), and using real-world bitcoin data (\cref{sec:bitcoin}). 

\subsection{A Family of Deceptively Simple SMDP}
\label{sec:sim}
We use simulated two-state SMDPs to empirically evaluate the difference between SMART and Relaxed SMART algorithms, and \harmonicr. 

\subsubsection{Simulation Setup}
\label{sec:sim-setup}
\cref{fig:state-machine} shows the small, continuing two-state SMDP environment that we use to evaluate the algorithms. The environment is constructed from states $s_1,s_2$ and actions $A,B$. From $s_1$, action $A$ deterministically transitions to $s_2$ with a reward $r_t$ and duration (sojourn time) $\tau_t$. Action $B$ behaves similarly but uses a different $(r_t,\tau_t)$ generator.  From $s_2$, the process deterministically returns to $s_1$ with reward $0$ and duration $1$, yielding a continuing task (no terminal states).
Each episode starts at $s_1$, and the reward/duration generators are reset at episode boundaries to ensure identical sampling across runs. 



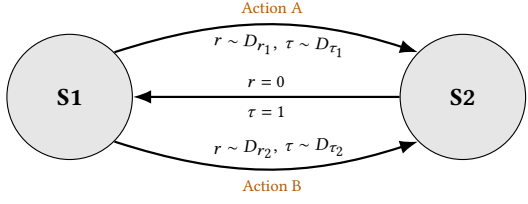
\begin{figure}[t]
\centering
\resizebox{0.82\linewidth}{!}{%
\begin{tikzpicture}[
  >=Latex,
  node/.style={circle, draw=black, fill=gray!20, minimum size=16mm, inner sep=0pt, font=\bfseries},
  lab/.style={font=\scriptsize},
  labA/.style={lab, text=orange!70!black},
  thickarrow/.style={->, line width=0.8pt}
]

\node[node] (s1) {S1};
\node[node, right=34mm of s1] (s2) {S2};

\draw[thickarrow]
  (s1.north east) to[out=18, in=162]
  node[pos=0.52, above=0.3mm, labA] {Action A}
  node[pos=0.55, below=0.2mm, lab, sloped] {$r\sim D_{r_1},\ \tau\sim D_{\tau_1}$}
  (s2.north west);

\draw[thickarrow]
  (s2.west) -- (s1.east)
  node[pos=0.50, above=0.0mm, lab] {$r=0$}
  node[pos=0.50, below=0.0mm, lab] {$\tau=1$};

\draw[thickarrow]
  (s1.south east) to[out=-18, in=-162]
  node[pos=0.52, below=0.3mm, labA] {Action B}
  node[pos=0.55, above=0.2mm, lab, sloped] {$r\sim D_{r_2},\ \tau\sim D_{\tau_2}$}
  (s2.south west);

\end{tikzpicture}
}

\caption{Two-state SMDP with stochastic reward and sojourn time for the transition $s1 \to s2$, and deterministic return for the transition $s2 \to s1$.}
\Description{Two-state SMDP with stochastic reward and sojourn time scheme}
\label{fig:state-machine}
\end{figure}

As shown in \cref{sec:rate-matters}, the algorithms' different approaches to estimating $\rho$ lead to identical results when the rewards and durations are independent of each other. 
We therefore focus on the case where the reward and durations are not independent. Furthermore, had the rate itself been stationary, then SMART would
already be sufficient, and thus we choose a non-stationary rate.

Specifically, we set up the rewards and durations such that action A may appear promising within the episode length, though its rewards and durations are independent. Action B may appear less promising initially, but its rewards and durations are not independent.  The objective of the algorithms is to discover that the actual expected rate of Action B is much better in the long term.

In the experiments, action A's rewards are taken from a function linear in $t$, with a slope of 0.05.  The durations of action A are stochastically sampled from the normal distribution $\mathcal{N}(\mu=1,\sigma=0.1)$, capped at 0.001 to prevent zero or negative durations.  The rewards and durations of action $A$ are clearly independent (in the long run).

Action B's rewards use a \textit{drifting log-scaled sine} (denoted \textit{SinLogD})---a drifting sine wave that scales and amplifies stronger as time passes (\cref{eq:sinlogd}). The durations use a \textit{drifting log-scaled cosine} (denoted \textit{CosLogD) (\cref{eq:coslogd}}):
\begin{align}
    SinLogD(t,\mathrm{offset},\mathrm{log\_scale}) =& \bigl(\sin(t)+\mathrm{offset}\bigr)\,10^{\,t\cdot \mathrm{log\_scale}}. \label{eq:sinlogd} \\ 
    CosLogD(t,\mathrm{offset},\mathrm{log\_scale}) =& \bigl(\cos(t)+\mathrm{offset}\bigr)\,10^{\,t\cdot \mathrm{log\_scale}}. \label{eq:coslogd}
\end{align}
  
We ensure the rewards and durations of action B are not independent by using a hyperparameter to set the $\mathrm{log\_scale}$ in the rewards and durations. Given a baseline value $l$, the reward function uses $\mathrm{log\_scale}=v$, and the duration function uses $\mathrm{log\_scale}=\dfrac{v}{2}$. 
For simplicity, the value of the offset throughout the experiment is set to 10 for both functions. While offset of zero still works, the addition of fast-iterating positive and negative values is hard to grasp visually, hence we keep the values positive only using the offset.

\cref{fig:SinLogD} shows the first thousand steps of each of the reward ($SinLogD$) and duration ($CosLogD$) functions and their ratio: 
$\frac{\operatorname{SinLogD}(t,10,0.001)}{\operatorname{CosLogD)}t,10,0.0005)}.$ The figure illustrates how quickly the non-stationary rewards and durations expand, while their non-stationary ratio expands much more slowly.

\begin{figure}[ht]
    \centering
    \includegraphics[width=0.75\linewidth]{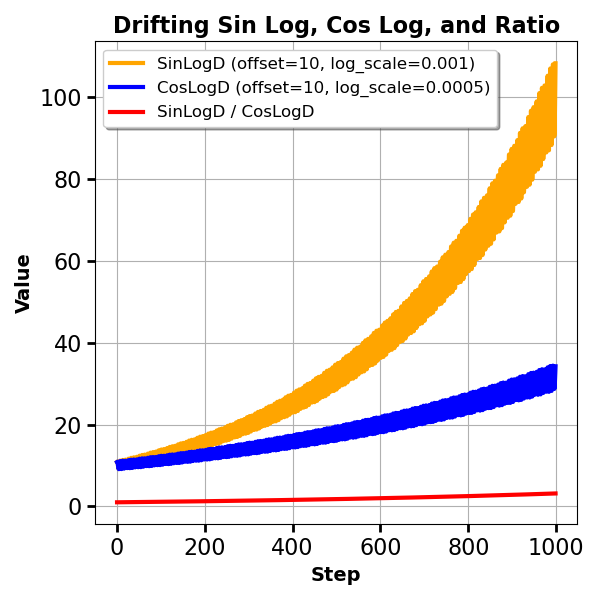}
    \caption{A thousand-step sample of $\operatorname{SinLogD}(t,10,0.001)$ and $\operatorname{CosLogD}(t,10,0.0005)$ with their ratio.}
    \Description{Line plots showing the first thousand steps of the drifting log-scaled sine (SinLogD) and cosine (CosLogD) functions and their ratio. The functions expand rapidly while their ratio grows much more slowly.}
    \label{fig:SinLogD}
\end{figure}

\cref{fig:ratio_at_10000} shows a specific setting, where action B's rewards and durations are defined as in \cref{fig:SinLogD}. The learning algorithms would be exposed to the first 1000 steps (shown as a dotted vertical line). Up to that point, action A seems superior. However, plotting ahead to 10,000 steps, we see that the rate resulting from action B overtakes that of action A. Action B is the optimal choice.  These settings are a good test for a learning algorithm's ability to uncover the underlying trend of the non-stationary average rate.


\begin{figure}[htbp]
    \centering
    \includegraphics[width=0.7\linewidth]{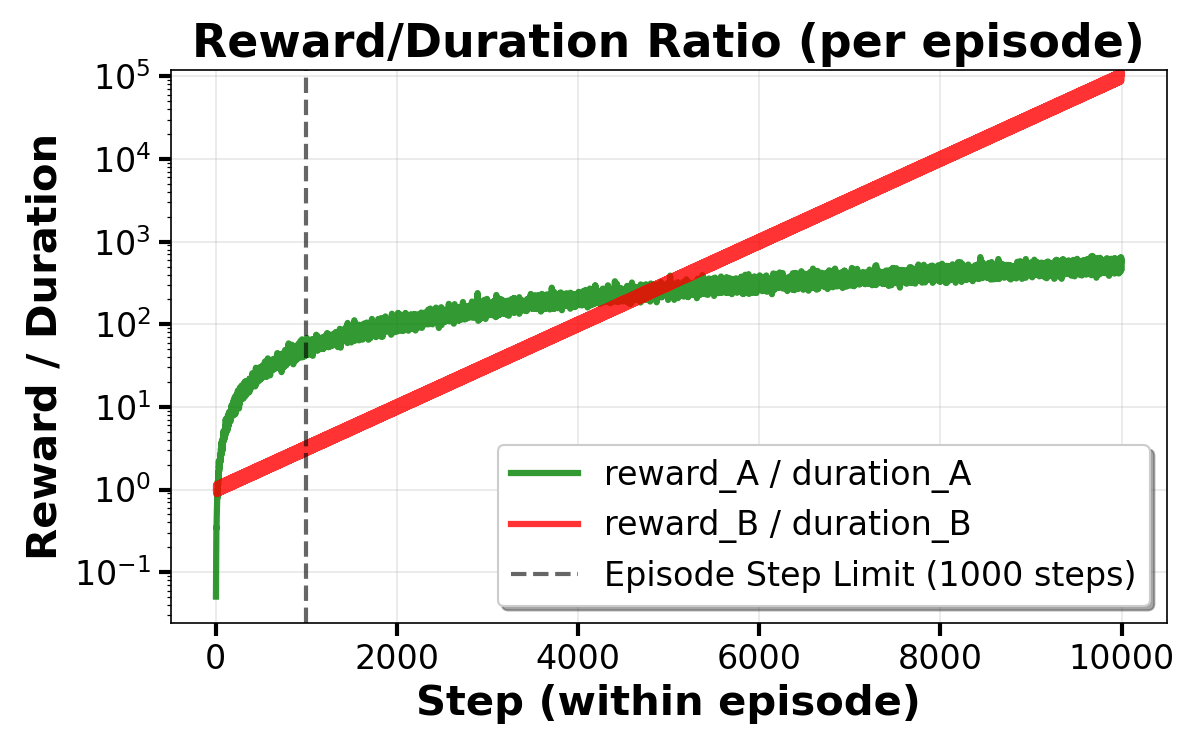}
    \caption{Ratio of $\operatorname{SinLogD}(t,10,0.001) / \operatorname{CosLogD}(t,10,0.0005)$  and $\operatorname{Linear}(0.05) / \operatorname{Normal}(1,0.1)$ for the first \textbf{$10{,}000$} steps. The dotted line shows the learning cutoff, meaning that the algorithms don't see the cross at $4{,}500$ where Action B starts to result in a better value than Action A.}
    \Description{A line plot over 10,000 steps comparing the reward-to-duration ratio of Action B (SinLogD/CosLogD) against that of Action A (Linear/Normal). A vertical black line marks the learning cutoff at step 1,000. Action A appears superior within the cutoff, but Action B overtakes it around step 4,500.}
    \label{fig:ratio_at_10000}
\end{figure}

In the experiments, we vary the difficulty of the settings by changing only the baseline $\mathrm{log\_scale}$ parameter controlling the relation between the rewards and durations of action B.  We test baseline values between $10^{-5}$ and $10^{-1}$ with 30 log-scale steps in this range. 

\cref{fig:ratio_at_10000_extra} shows that smaller values lead to more difficult problems. $\mathrm{log\_scale}$ values greater than 0.003 are ``easier'' because Action B dominates within the $1{,}000$ steps of the learning. Smaller values make the settings harder (moving the crossover point further and further away). 

\begin{figure}[htbp]
    \centering

\begin{subfigure}[b]{0.48\linewidth}
  \centering
  \includegraphics[width=\linewidth]{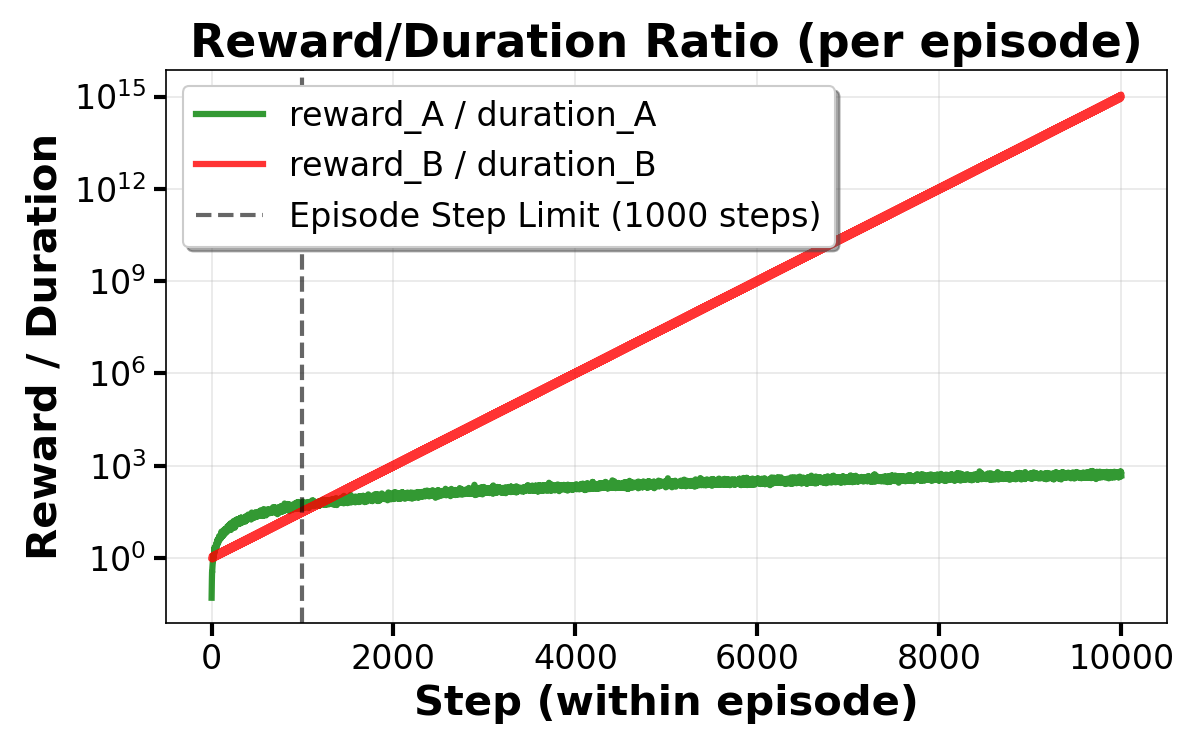}
  \caption{$\operatorname{SinLogD}(t,10,0.003)\,/$ \\ $\operatorname{CosLogD}(t,10,0.0015)$}
    \end{subfigure}
    \hfill
    \begin{subfigure}[b]{0.48\linewidth}
      \centering
      \includegraphics[width=\linewidth]{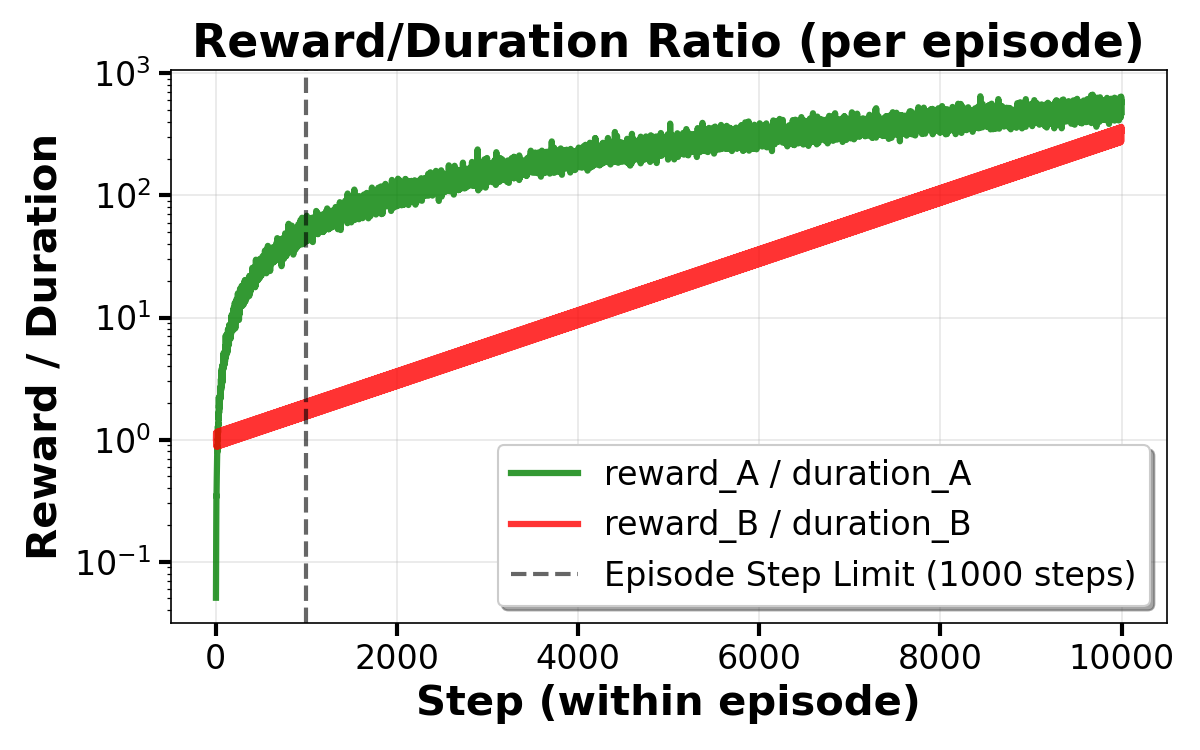}
      \caption{$\operatorname{SinLogD}(t,10,0.0005)\,/$ \\ $\operatorname{CosLogD}(t,10,0.00025)$}
    \end{subfigure}
        \caption{Different log scales in relation to $\operatorname{Linear}(0.05) / \operatorname{Normal}(1,0.1)$ for the first \textbf{$10{,}000$} steps. As the log scale value lowers---e.g., (b) above---the cross-over point takes place later, at a step unobserved by the learning algorithm.}
    \Description{Two side-by-side line plots comparing reward-to-duration ratios over 10,000 steps for two different log-scale settings. The left plot uses a log scale of 0.003 and the right uses 0.0005. Lower log-scale values push the crossover point beyond the learning cutoff, making the task harder.}
    \label{fig:ratio_at_10000_extra}
\end{figure}

Each algorithm is run on 4 episodes, each of 1000 steps, with a fixed exploration rate of 0.2.  To assure ourselves of some robustness to the learning rate parameters $\alpha, \beta$, we use $\alpha\in [10^{-4}, 0.1]$ and $\beta\in [10^{-4},10^{-1}]$. For both, we use 20 even log-scale steps, i.e., a total of 400 $\alpha, \beta$ combinations. Note that SMART does not use a $\beta$ parameter. 

\subsubsection{Simulation Results}
\label{sec:sim-res}
\cref{fig:robustness} plots the results from the experiments. The horizontal axis measures the difficulty of the SMDP environment, using a log-scale. Each point along the axis is a baseline value used for the $\mathrm{log\_scale}$ parameter for action B, as described. As we move from left to right, the learning settings become harder and harder, as the optimality of action B is not observable by the algorithms, and the cross-over point is further away in the horizon.  The vertical axis measures the percentage of successful trials to learn the optimal policy (action B), out of the 400 attempts (with different $\alpha, \beta$ parameters). 

The figure shows that \harmonicr is more robust and successful in these experiments. It is shown to be more robust, because its success percentage is higher in all cases (i.e., it is less sensitive to specific $\alpha,\beta$ learning rates).  It is shown to be more successful, as it continues to maintain a high percentage of successes, even as the environment becomes too difficult for SMART and Relaxed SMART, which drop to 0\% success around the middle of the horizontal axis.

\begin{figure}[h]
    \centering
    \includegraphics[width=1\linewidth]{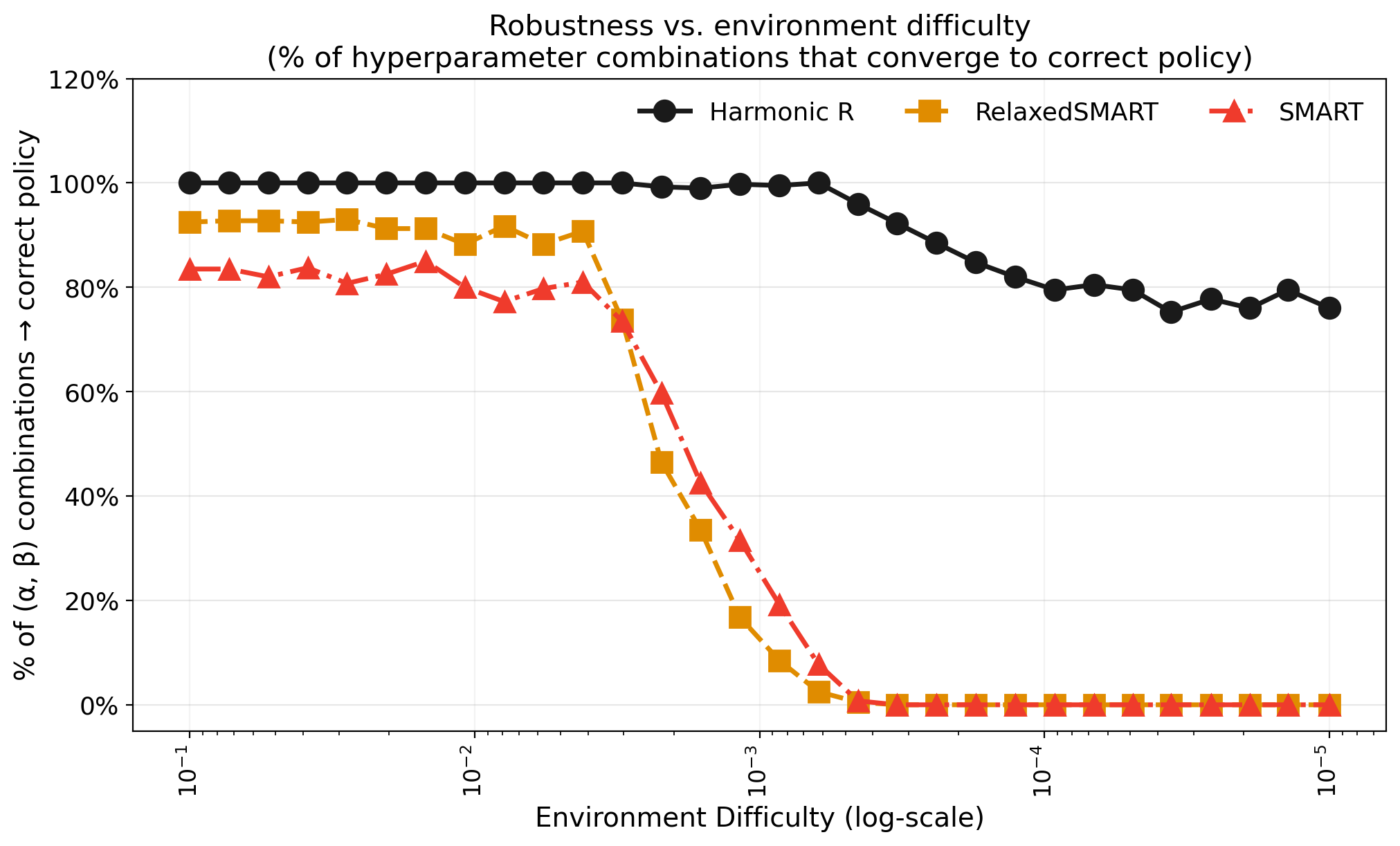}
    \caption{Comparing the success rate of the different algorithms, as learning problems become harder (left to right). The \harmonicr algorithm clearly dominates (shown in black).}
    \Description{A line chart showing the success rate of Harmonic R-Learning, SMART, and Relaxed-SMART algorithms across varying log-scale difficulty levels. Harmonic maintains a higher success rate at lower (harder) log-scale values, where SMART and Relaxed-SMART rapidly deteriorate.}
    \label{fig:robustness}
\end{figure}

These results are very promising, but this simple environment has only positive values and is artificially structured.
In the next set of experiments, we deal with a challenging environment that was not constructed manually and has negative and zero values. 

\subsection{Bitcoin environment}
\label{sec:bitcoin}
We test the algorithms on highly volatile, generally non-stationary data: real-world Bitcoin (BTC) market trading data. Bitcoin is uniquely suitable: it trades continuously (24/7) with high liquidity and dense minute-level observations, enabling rate-sensitive, infinite-horizon analysis without episodic boundaries. Publicly available data provide over 500{,}000 one-minute observations per year, with gapless bars where each close equals the next open\footnote{Data source: \url{https://www.kaggle.com/datasets/swaptr/bitcoin-historical-data}, licensed for open use.}.

BTC exhibits documented non-stationary features: heavy tails, volatility clustering, regime shifts, and explicit ruin barriers (which would adversarially stress robustness)~\cite{drozdz2018bitcoin,dehouche2021scalemattersdailyweekly,tang2025stylizedfactshighfrequencybitcoin,bodek2013problem}. These properties make simulated Bitcoin trading a very demanding RL testbed. In addition, publicly-available data permits precise ablations.

\subsubsection{Experimental Setup}\label{sec:setup}
We compare the behavior of the three algorithms: Relaxed-SMART, SMART, and \harmonicr. As a reminder, SMART addresses rates, but uses the long-term sample arithmetic average to approximate the expected optimal reward rate $\rho$---brittle when rewards are non-stationary. Relaxed-SMART uses EMA to handle non-stationarity, but bases the approximation on the ratio of averages, which assumes the rewards and durations are independent.  The arithmetic mean. \harmonicr uses the stochastically estimated \hname (\cref{eq:harmonic-update-rho}) to target the rate directly.

\textbf{Data Preparation.}
We use the entire set of minute-by-minute BTC historical prices (2012 to mid-September 2025). To eliminate dependency on starting conditions, we partitioned the data into 21 non-overlapping segments, each (except for the last one) with 350{,}000 one-minute observations. This supports walk-forward evaluation without look-ahead leakage.

At each decision point (every minute), the algorithm is given the choice of either buying a bitcoin or selling one.  It gets a reward depending not only on the direction of the decision, but also on the difference between the price at the moment the action was executed and the price at the end of the one-minute interval.  

An MDP version of this task would assume a fixed duration of action execution, e.g., at every minute, a decision is made, and the action is taken. So, for instance, the reward depends on the entire one-minute price difference.  An SMDP version of this task would instead account for the duration of the action itself.  If an action takes $\tau$ seconds, the reward is computed by the difference between the end price and the price $\tau$ seconds from the 1-minute interval beginning (determined in these experiments by a linear interpolation between the opening and closing minute prices). Quicker actions thus use more of the 1-minute price difference.

We generate two SMDP versions of the data.  In the \emph{random} version, the duration of actions is selected uniformly from the interval of 5 to 45 seconds. In the second version (\emph{scaled}), the duration is \textit{scaled} by the value of the reward using min-max scaling according to the minimum and maximum reward values. Larger rewards impose larger durations, and vice versa (within  5--45 seconds).

\textbf{Experiment Parameters.}
At each timestamp, the value of BTC can either increase or decrease. Therefore, the state space is defined as a window capturing the recent trend directions of BTC. For example, with a state size of three, the state represents the last three movement directions of BTC, such as $\{up, up, down\}$ or $\{down, up, up\}$.  We investigate four values: 3, 6, 9, and 12.

We sweep over combinations of state sizes (affecting all algorithms) and mean smoothing parameter \(\beta\) with values: 0.01, 0.05, and 0.1 (affecting only Harmonic and Relaxed-SMART). Each (state, \(\beta\)) pair is evaluated independently on all 21 segments for both delay time windows. 
All algorithms used an exploration rate of 0.2 with a $\epsilon$ decay rate of 0.999 and a learning rate of 0.001.  For each of the 21 files and their 350{,}000 data points, every algorithm was run with 30 different seeds.

\textbf{Evaluation Metric.}
We measure the accumulating rewards (profits and losses) for each algorithm instance, at each step. The final result of the accumulated value is therefore the accumulated reward over an entire segment. The results from 30 runs are averaged and used to represent the performance of each agent.

For each window size, method, and (state, $\beta$), we aggregate the values across all agents, keeping the mean and std of the accumulated reward at each timestamp. Only on-policy actions are taken into account, any exploration action is ignored for fairness.

\subsubsection{Results}\label{sec:results}




To compare all $\beta$ and state size values for all segments, we summarize the results as heatmaps in two axes: the state size and the $\beta$ values. These are shown in \cref{fig:H_vs_all_random} for the random SMDP version, and in \cref{fig:H_vs_all_scaled} for the scaled SMDP version.  In both figures, subfigure (a) shows the heatmap of wins (percentage in which \harmonicr was better, out of the 21 cases) against SMART, and subfigure (b) shows the same against Relaxed SMART.  Values around 0.5 hint at equivalent performance.


\begin{figure}[htbp]
    \centering

    \begin{subfigure}[b]{0.49\linewidth}
        \centering
        \includegraphics[width=\linewidth]{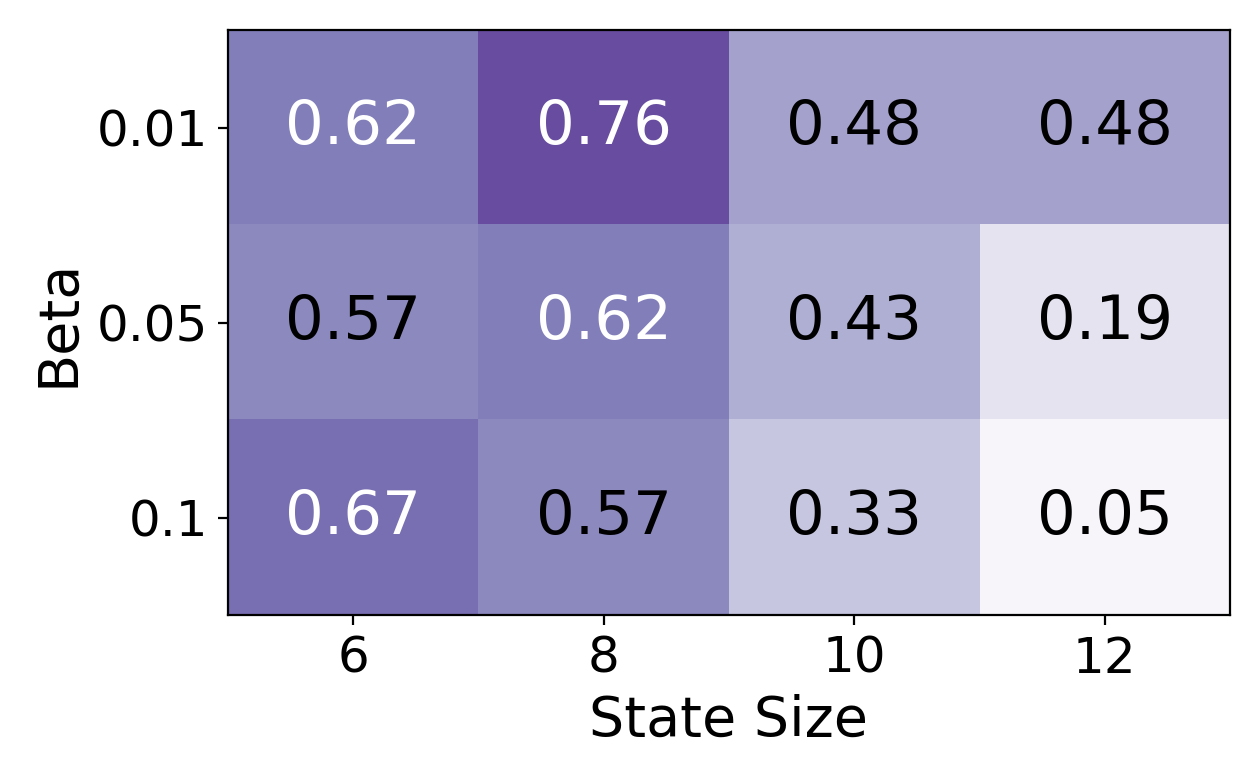}
        \caption{Harmonic vs SMART}
    \end{subfigure}
    \hfill
    \begin{subfigure}[b]{0.49\linewidth}
        \centering
        \includegraphics[width=\linewidth]{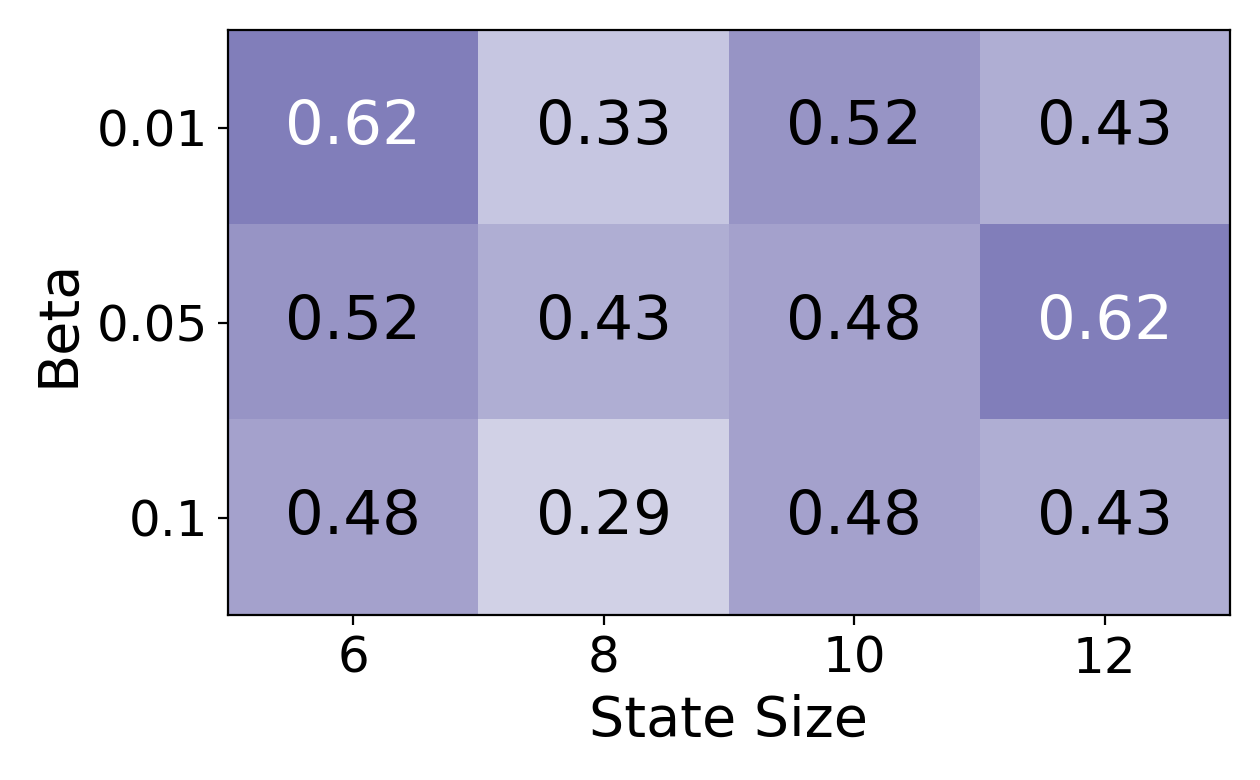}
        \caption{Harmonic vs Relaxed-SMART \\}
    \end{subfigure}

    \caption{Win-Ratio of harmonic for all 21 segments of the \textit{random} version. The rewards and durations are independent. Harmonic is not performing better than SMART and Relaxed-SMART.}
    \Description{Two heatmaps showing the win ratio of Harmonic R-Learning against SMART (left) and Relaxed-SMART (right) across state sizes and beta values in the random-delay environment. The ratios are near 0.5, indicating no significant advantage for any algorithm when rewards and durations are independent.}
    \label{fig:H_vs_all_random}
\end{figure}
\vspace{-10pt}

\begin{figure}[htbp]
    \centering

    \begin{subfigure}[b]{0.49\linewidth}
        \centering
        \includegraphics[width=\linewidth]{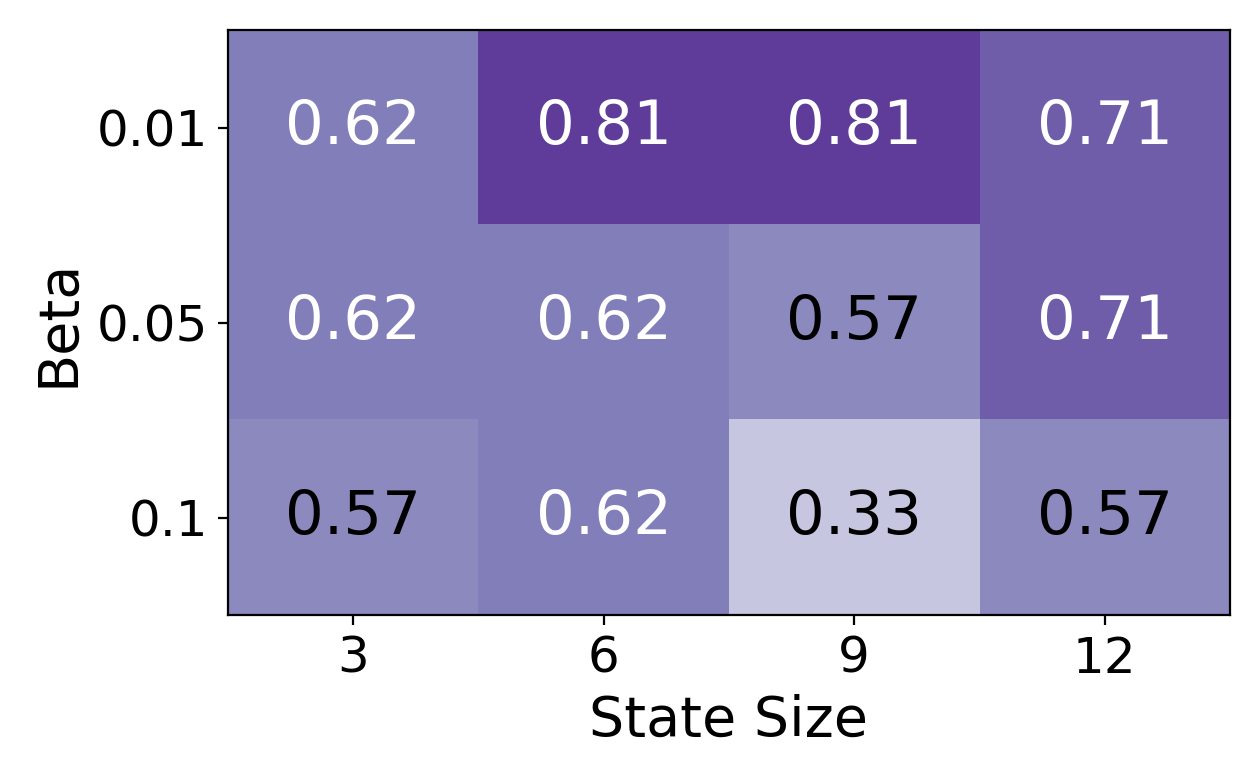}
        \caption{Harmonic vs SMART}
    \end{subfigure}
    \hfill
    \begin{subfigure}[b]{0.49\linewidth}
        \centering
        \includegraphics[width=\linewidth]{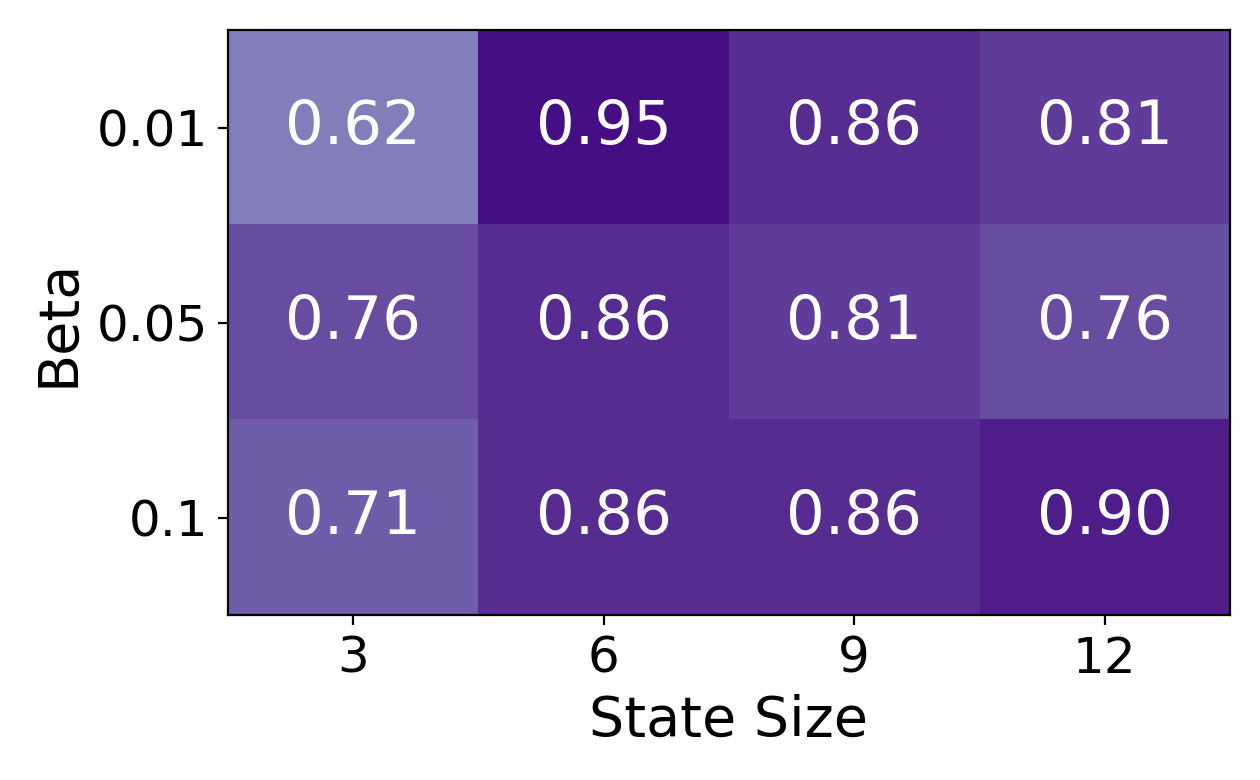}
        \caption{Harmonic vs Relaxed-SMART}
    \end{subfigure}

    \caption{Win-Ratio of harmonic for all 21 segments of the \textit{scaled} version. The dependency of the reward and time gives the Harmonic algorithm a significant advantage, creating a clear win in most cases.}
    \Description{Two heatmaps showing the win ratio of Harmonic R-Learning against SMART (left) and Relaxed-SMART (right) across state sizes and beta values in the scaled-delay environment. Harmonic shows a clear and consistent advantage, with win ratios predominantly above 0.5 across most configurations.}
    \label{fig:H_vs_all_scaled}
\end{figure}


\vspace{-8pt}
\section{Conclusions and Future Work}\label{sec:fin}
This paper makes several contributions. First, it analyzes the requirements for an average rewards learning algorithm for semi-Markov decision processes. It shows that SMART and Relaxed-SMART, the only algorithms for this task, as a result of using the sample average, can fail when rewards are not in a linear relation to time. To address this, we introduce a \hname, a general mean operator that generalizes the harmonic mean where the latter is applicable, and obeys key axiomatic properties of general means. We proved the generalization and the inclusion of the \hname\ in the class of general means. We also proved that it is not a member of the quasi-arithmetic means class.  As the \hname\ generalizes the harmonic mean, it is mathematically correct for use in averaging reward rates, as needed.  We demonstrate empirically, using a proof-of-concept simulation and a highly volatile real-world Bitcoin trading data (millions of data points), that use of \harmonicr
leads to improved results over existing algorithms for average reward maximization.

\bibliographystyle{ACM-Reference-Format}
\bibliography{full,finrl,harmonic}


\appendix
\section{Proof of Theorem ~\ref{thm:3collections-dependent}}
\label{sec:full-proof}


In real world scenarios, we need an average rate method that can extend to treat scenarios of: positive, negative and zero value rewards with non zero time. Extending our proof of Theorem \ref{thm:noconsistency}, below we show that this proof holds when the joint distribution of rewards (where positive, negative and zero values are commingled) is decomposed into its three components, as needed to generalize equation \ref{eq:harmonic-mixed}. To do this, we work with $\sigma$-algebras in this proof. By way of brief mention, a $\sigma$-algebra is a neat mathematical representation of a set of variables where the requirement that once we can show a property about a certain variable (such as reward or time), we can then extrapolate it to also talking about any combinations of variables (e.g., “all rewards" or "all time periods"). In this way, the $\sigma$-algebra defines a property (or properties) of a collection of events where unbounded random variables $r$ and $\tau$ generate the event collections $\sigma(X)$ and $\sigma(T)$ that contain all measurable statements about $r$ and $\tau$, respectively.  In this way this extends to infinite horizon R-learning tasks naturally.

The motivation for this proof is to show that, even when harmonic means are decomposed into their mutually exclusive partitions based on sign, the resulting distributions of rewards and time, retain any dependence structure (or more formally, remain not independent). This lingering dependence structure in effect, guarantees that harmonic mean will exhibit a  $\rho$ value that differs from that produced by SMART and Relaxed-SMART even when decomposed into a calculation using the positive, negative and zero components of $r$ from equation \ref{eq:harmonic-pos-neg}.  Hence we posit, as we did in stating \cref{thm:3collections-dependent}: \textit{if the distributions of all rewards is \textbf{not independent} of the distribution of all holding times, then when rewards are partitioned based on their sign (positive, negative, zero), at least one of three reward partition collections must also be dependent on the respective partition of holding times.}

To expand further consider the following. 
Let $(\Omega,\mathcal{F},\mathbb{P})$ be a probability space. Let
    \[
        r:\Omega\to\mathbb{R},\qquad \tau:\Omega\to\mathbb{R}
    \]
be random variables (with arbitrary distributions), where $r$ denotes the rewards and $\tau$ denotes the holding time.
In an SMDP, a single execution of the policy from one decision epoch to the next produces a random \emph{experience segment} (also called a sojourn), e.g.\ a tuple
    \[
    \omega := (s_0,a_0,r_0,\tau_0, \; s_1,a_1,r_1,\tau_1, \; \ldots,\; s_{N}, a_{N}, r_{N}, \tau_{N}) 
    \]

containing the visited states, chosen actions, intermediate rewards, and elapsed times until the next decision epoch (or termination). The set $\Omega$ is the collection of all such possible segments $\omega$; in other words, $\Omega$ is the sample space of all outcomes that could be generated by the environment dynamics together with the (possibly stochastic)
policy.

The $\sigma$-algebra $\mathcal{F}$ is the collection of events (subsets of $\Omega$) to which we assign probabilities, such as ``the segment terminates within $2$ seconds'' or
``the total reward exceeds $10$''. The probability measure $\mathbb{P}$ is the distribution over segments induced by the environment's stochastic transition-and-holding-time kernel together with the policy.

A random variable is a measurable function of the realized segment $\omega\in\Omega$. In particular, we define
    \[
    r:\Omega\to\mathbb{R},\qquad \tau:\Omega\to(0,\infty)
    \]
by letting $r(\omega)$ be the total reward accrued during the segment $\omega$ (e.g.\ cumulative reward until the next decision epoch) and letting $\tau(\omega)$ be the corresponding holding time (elapsed physical time) of that segment.

Now, assume we decompose the rewards space into three disjoint collections (positive, negative and zero rewards). Formally, let
    $$
    H_{+},\,H_{-},\,H_{0}\subseteq \mathbb{R},
    $$
    $$
    \qquad
    H_{+}\cap H_{-}=\emptyset,
    H_{+}\cap H_{0}=\emptyset,
    H_{-}\cap H_{0}=\emptyset,
    H_{+}\cup H_{-}\cup H_{0}=\mathbb{R}.
    $$

Define the three reward-collection events
    \[
    A_{+}:=\{r\in H_{+}\},\qquad A_{-}:=\{r\in H_{-}\},\qquad A_{0}:=\{r\in H_{0}\}.
    \]

Equivalently, define the reward-collection label
    \[
        Z:\Omega\to\{+,-,0\},
        \qquad
        Z(\omega)=
        \begin{cases}
        + & \text{if } r(\omega)\in H_{+},\\
        - & \text{if } r(\omega)\in H_{-},\\
        0 & \text{if } r(\omega)\in H_{0}.
        \end{cases}
    \]

Then, if the reward-collection label is not independent of time, i.e.

    \[
    Z \not\!\perp\!\!\!\perp T .
    \]

it follows that \emph{at least one} of the three collections is itself not independent of time. That is, there exists
$\kappa\in\{+,-,0\}$ and there exists an event $B\in\sigma(T)$ such that
\[
\mathbb{P}(A_{\kappa}\cap B)\neq \mathbb{P}(A_{\kappa})\,\mathbb{P}(B).
\]

\textit{Note:} the proof below relies on the property of \emph{finite additivity} of probability, which is a direct consequence of Kolmogorov's axioms (Kolmogorov, 1933): for pairwise disjoint events $E_1,\dots,E_m\in\mathcal{F}$,
\[
\mathbb{P}\!\left(\bigcup_{j=1}^{m} E_j\right)=\sum_{j=1}^{m}\mathbb{P}(E_j).
\]

\begin{proof}
We prove by contradiction. Assume that $Z \ne\ \perp\!\!\!\perp\ T$, but that nevertheless
\emph{each} of the three (disjoint) reward collections, is independent of time. In $\sigma$-algebra notation, this means that for every
\[
B\in\sigma(T),
\]
we have

\begin{align}
\label{eq:each-collection-ind}
\qquad
\mathbb{P}(A_{+}\cap B)=\mathbb{P}(A_{+})\mathbb{P}(B) \\
\mathbb{P}(A_{-}\cap B)=\mathbb{P}(A_{-})\mathbb{P}(B) \\
\mathbb{P}(A_{0}\cap B)=\mathbb{P}(A_{0})\mathbb{P}(B)
\end{align}

Now, note that $\sigma(Z)$ is generated by the three atoms $A_{+},A_{-},A_{0}$. In particular, every event
$C\in\sigma(Z)$ can be written as a union of a subcollection of these three events. That is, for every $C\in\sigma(Z)$
there exists a set $S\subseteq\{+,-,0\}$ such that
\begin{equation}\label{eq:C-as-union}
C=\bigcup_{\kappa\in S} A_{\kappa}.
\end{equation}

Fix any $C\in\sigma(Z)$ and any $B\in\sigma(T)$. Using \eqref{eq:C-as-union} and the distributive property of intersection,
\[
C\cap B
=
\left(\bigcup_{\kappa\in S} A_{\kappa}\right)\cap B
=
\bigcup_{\kappa\in S} (A_{\kappa}\cap B).
\]
Since the events $A_{+},A_{-},A_{0}$ are disjoint, the events $(A_{\kappa}\cap B)$ are also disjoint across $\kappa\in S$.
Therefore, by finite additivity (Kolmogorov, 1933),
\[
\mathbb{P}(C\cap B)
=
\sum_{\kappa\in S}\mathbb{P}(A_{\kappa}\cap B).
\]
Applying \eqref{eq:each-collection-ind} to each term,
\[
\mathbb{P}(C\cap B)
=
\sum_{\kappa\in S}\mathbb{P}(A_{\kappa})\mathbb{P}(B)
=
\mathbb{P}(B)\sum_{\kappa\in S}\mathbb{P}(A_{\kappa}).
\]
Finally, since the $A_{\kappa}$ are disjoint, finite additivity gives
\[
\sum_{\kappa\in S}\mathbb{P}(A_{\kappa})
=
\mathbb{P}\!\left(\bigcup_{\kappa\in S} A_{\kappa}\right)
=
\mathbb{P}(C),
\]
and therefore
\begin{equation}\label{eq:sigmaZ-sigmaT-ind}
\mathbb{P}(C\cap B)=\mathbb{P}(C)\mathbb{P}(B)
\qquad
\text{for all } C\in\sigma(Z)\ \text{and all } B\in\sigma(T).
\end{equation}

Equation \eqref{eq:sigmaZ-sigmaT-ind} is exactly the definition of independence of the $\sigma$-algebras $\sigma(Z)$ and
$\sigma(T)$, i.e.
\[
\sigma(Z)\ \perp\!\!\!\perp\ \sigma(T),
\]
which implies
\[
Z\ \perp\!\!\!\perp\ T.
\]
This contradicts our assumption that $Z \ne\ \perp\!\!\!\perp\ T$. Hence, our supposition that all three collections are
independent of time must be false. Therefore, at least one of the three collections must be dependent on time, i.e.\ there exists
$\kappa\in\{+,-,0\}$ and $B\in\sigma(T)$ such that
\[
\mathbb{P}(A_{\kappa}\cap B)\neq \mathbb{P}(A_{\kappa})\,\mathbb{P}(B).
\]
Q.E.D

\end{proof}

\end{document}